\documentclass[letterpaper, 10pt, conference]{IEEEconf}  

\IEEEoverridecommandlockouts                              %

\usepackage{algpseudocode}
\usepackage{algorithm}
\usepackage{graphicx}
\graphicspath{{.}, {./figs/}}
\usepackage{subcaption} 
\usepackage{mathptmx}
\usepackage{amsmath}
 
\usepackage{amsthm}
\usepackage{amssymb}
\usepackage{mathtools}
\usepackage{xspace}
\usepackage{tabularx}
\usepackage{mathrsfs} 
\usepackage[usenames,dvipsnames,svgnames,table]{xcolor}
\usepackage{tikz}

\newenvironment{defitems} {
    \begin{list}{$-$}{%
        \setlength{\leftmargin}{16pt}
        \setlength{\topsep}{-4pt}
        \setlength{\partopsep}{0pt}
        \setlength{\itemsep}{2pt}
        \setlength{\itemindent}{-6pt}}
        \ignorespaces}
{\unskip\end{list}}

\providecommand{\mdp}{\textsc{mdp}\xspace}
\providecommand{\mdps}{\textsc{mdp}s\xspace}
\providecommand{\pomdp}{\textsc{pomdp}\xspace}
\providecommand{\pomdps}{\textsc{pomdp}s\xspace}
\providecommand{\psomdp}{\textsc{pso-mdp}\xspace}
\providecommand{\psomdps}{\textsc{pso-mdp}s\xspace}
\providecommand{\expectation}[1]{\ensuremath{\mathbb{E}\left(#1\right)}}

\providecommand{\instance}[1]{\ensuremath{\mathscr{#1}}}

\providecommand{\gobble}[1]{}

\providecommand{\checkin}{\ensuremath{\kappa}\xspace}
\providecommand{\comp}{\ensuremath{\textrm{cmp}}}

\providecommand{\opt}{\ensuremath{\ast}}

\providecommand{\nop}{\textsc{no-op}\xspace}
\providecommand{\anop}{\ensuremath{a_{\bot}}}
\providecommand{\action}{\ensuremath{\vec a}}
\providecommand{\actions}{\textsc{A}}
\providecommand{\actionsuffix}{{\vec{u}}}
\providecommand{\state}{\ensuremath{s}}
\providecommand{\states}{\textsc{S}}
\providecommand{\algoactions}[1]{\textsf{\small A}({#1})}
\newtheorem{theorem}{\textbf {Theorem}}
\newtheorem{lemma}[theorem]{\textbf {Lemma}}
\newtheorem{definition}{\textbf{Definition}}
\newtheorem*{notation}{\textbf{Notation}}

\providecommand{\Reals}{\ensuremath{\mathbb{R}}}

\providecommand{\Nats}{\ensuremath{\mathbb{N}_{>0}}}
\providecommand{\expectation}[1]{\ensuremath{\mathbf{E}\left(#1\right)}}

\newif\ifmargincomments %
\margincommentsfalse

\ifmargincomments

\else

\fi

\newif\ifrelaxedv  %
\relaxedvfalse
\ifrelaxedv
\newcommand{\jvspace}[1]{}
 \else
\newcommand{\jvspace}[1]{\vspace{#1}}
\renewcommand{\baselinestretch}{.985}
\fi

\title{\huge \bf
Planning under periodic observations: bounds and bounding-based solutions 
\vspace*{-0.9ex}
}

\author{Federico Rossi$^{1}$ and Dylan A. Shell$^{2}$%
\thanks{$^{1}$Jet Propulsion Laboratory, California Institute of Technology, Pasadena, CA 91109, USA.
        {\tt\footnotesize federico.rossi@jpl.nasa.gov}}%
\thanks{$^{2}$Dept. of Comp. Sci. \& Eng., Texas A\&M University, College Station, TX 77843, USA.
        {\tt\footnotesize dshell@tamu.edu}}%
}

\begin{document}

\maketitle
\thispagestyle{empty}
\pagestyle{empty}

\begin{abstract}

We study planning problems faced by robots operating in uncertain environments
with 
incomplete knowledge of state, and
actions that are noisy and/or imprecise.
This paper identifies a new problem sub-class that models
settings in which information is revealed only intermittently through
some exogenous process that provides state information periodically.
Several practical domains fit this model, including the specific scenario that motivates our research:  
autonomous navigation of a planetary exploration rover augmented by remote imaging. 
With an eye to efficient specialized solution methods, we examine the structure of instances of this sub-class. %
They lead to Markov Decision Processes with exponentially large action-spaces but for which,
as those actions comprise sequences of more atomic elements,
one may establish 
performance
bounds by comparing policies under different information assumptions.
This provides a way in which to construct performance bounds systematically.
Such bounds are useful because, in conjunction with the insights they confer, they can be employed in bounding-based methods to obtain high-quality solutions efficiently; the empirical results we present demonstrate their effectiveness for the considered problems.
The foregoing has also alluded to the distinctive role that time plays for these problems\,---more specifically: time until information is revealed---\,and we uncover and discuss several interesting subtleties in this regard.
\end{abstract}

\section{Introduction}

Autonomous robots are compelled to cope with uncertainty.  The inherent
imperfections of sensing and actuation, as well as the inevitable shortfalls of
world models, mean that robots must select actions 
despite having only imprecise state information.  Unfortunately,
as is well known, the problem of planning under uncertainty in full generality
remains out of practical reach---except in problem instances that are tiny or where planning horizons are short.  In light of
this predicament, this paper represents a campaign of attack focused on
specialization: it aims at uncovering opportunities for development of
efficient methods that produce high quality solutions, even if only
for a restricted sub-class of planning problems. So long as the sub-class includes problems of practical 
value, such methods will have obvious utility.
As motivation, we begin with a specific instance of signal interest to us.

\begin{figure}
  \vspace*{-1ex}
  \begin{center}
    \includegraphics[scale=0.425]{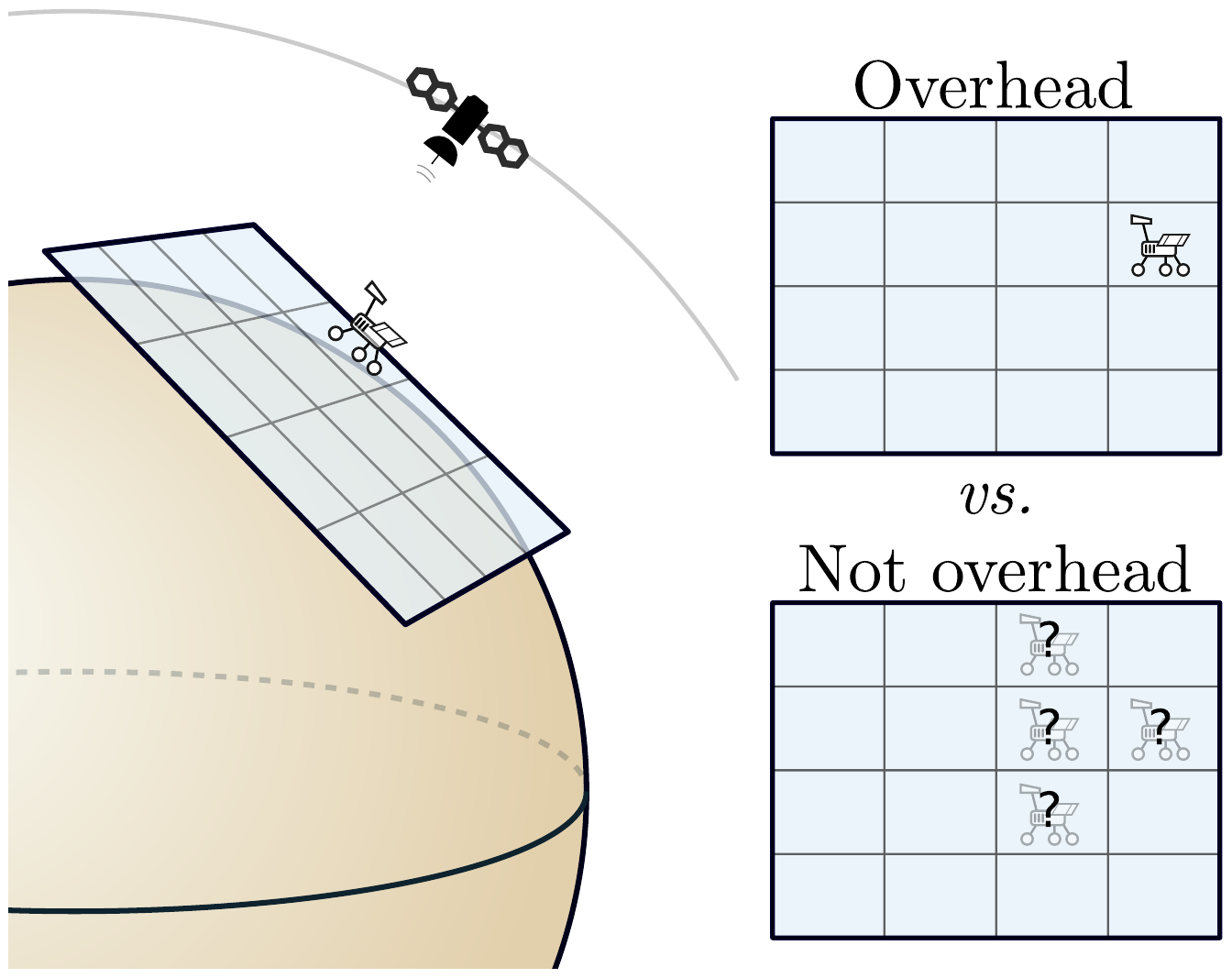}
  \vspace*{-1ex}
  \end{center}
  \caption{Motivating example: a satellite orbiting a planetary body helps 
localize an autonomous rover that is tasked with operating on the body's surface. The
rover executes a sequence of actions, but only obtains its state (shown as a
definitive location within a specific cell) when the satellite is overhead.
The rover must plan and act despite receiving observations that only arrive
periodically.
  }
  \vspace*{-0.2ex}
  \label{fig:roversatellite}
\end{figure}

Consider the autonomous rover in Figure~\ref{fig:roversatellite} that is navigating across the surface of some remote asteroid, moon, or planetoid.
Its objective is to reach a goal region efficiently in order to collect samples at that location for detailed analysis later.  
Even if it departs from a known position, the rover's knowledge of its pose rapidly becomes unreliable unless sensors can help circumscribe probable locations.
Suppose that, along with the rover,
 a separate orbital device had also been deployed.  
This satellite carries surface-directed sensors that include a detector capable of picking up the rover.  
From its extrinsic perspective, as the satellite circles, it acquires  information (e.g., imagery and ranging data) providing the rover's position.  
When the two are in communication range, there is the possibility of a \emph{check-in} to provide the rover with its  location.

In this scenario, the rover's knowledge of its state is sporadic: the
data providing its pose are sparse, though regular, and
when the check-ins do occur they resolve the rover's position. 
From the point of view of
the rover, the process that generates observations is exogenous.  The process's
periodicity is known, which means that, even though the rover may not know the  
information it will receive (since it does not know where it is, precisely), it
can be certain \emph{when} the data will be received.  

The traits present in the rover example\,---\emph{viz.} infrequent but periodic
observations of state---\,form a special sub-class of partially-observable
planning problems. 
These same properties also appear in other robotic domains.
For instance, marine robots operating in tidal regions may find their sensors
inhibited by periodic phenomena (e.g., those driven by diurnal factors).
Quite different instances arise when, to reduce the energy expended on radio
transmission, a team of multiple robots employs a pre-determined
synchronization and communication schedule. 
In fact, intermittency can have multiple advantages, such as in
facilitating stealthy operation, desirable for 
robot operation in clandestine conditions.

One significant source of complexity in dealing with general
partially-observable problems is that they involve balancing information
gathering with reward-realizing activities.  Being planning problems, there are
(state- or belief-mediated) correlations across time which complicate the
process of choosing actions.  Any state-revealing observation process
that is exogenous weakens what otherwise would be a tightly-coupled causal cycle.  In problems
like that in Figure~\ref{fig:roversatellite}, the robot's actions affect
neither when it will receive an observation nor the quality  of its estimate
when it does; none of its actions can be said to gather more information than
any other.  But the cycle is not entirely severed either, because the actions do
affect what is observed, i.e., what state the robot finds itself in when a check-in occurs.
Critically, when the observation process is periodic, the robot may select
actions\,---knowing when observations will arrive---\,so that it is at some
juncture where what is discerned will be of most value.

\subsection{Contribution and Organization}
Our contribution is threefold.
First, in Section \ref{sec:definitions}, we formally define the novel problem of decision-making under uncertainty with periodic information check-ins  
 as a stochastic decision problem with the usual assumption of an underlying
Markovian process. 
Second, in Sections~\ref{sec:ub} and~\ref{sec:lb}, we derive upper and lower bounds for the state-action values of the problem that can be computed efficiently. 
Along the way, Section~\ref{sec:ub:check-in-counter-example} presents an important example
illustrating that more frequent observations are not always better.
Third, in Section~\ref{sec:bnb} we propose a branch-and-bound method that makes use of
bounds to compute exact solutions to the PSO-MDP;
Section~\ref{sec:num-ex} presents numerical results that testify
to the effectiveness of this algorithm, showing that it is significantly faster compared to a naive MDP approach. 
The final section, Section~\ref{sec:concl},
presents our conclusions. %

\subsection{Related work}

The framework of Markov Decision Processes (\mdps) is a useful basis for optimal
control, planning, and learning in
robots~\cite{bertsekas19reinforcement,sutton18reinforcement}.
Classic fully-observable problems have a rich history with a variety of
effective solution techniques~\cite{bellman1984history}; recent work has sought
extensions to the basic \mdp formulation to capture additional features
including time-varying models~\cite{boyan2000exact} and more complex
representations~\cite{diuk2008object}, as well as exploring various means to
improve performance, especially in solving very large
instances~\cite{swiechowski21}.  
The settings we will consider are not fully-observable (except in the degenerate case with unit period),
and it so might be better considered partially observable.

The general framework of Partially-Observable Markov Decision Processes
(\pomdps) has been explored as a solution to robotics problems for the last two
decades~\cite{kurniawati22pomdp}; %
early attempts to apply the techniques of
the day ended up highlighting the twin curses of history and
dimensionality~\cite{pineau2003point} as obstructions to the
tractable solution of \pomdps.  A subsequent and popular line of work then
pursued policy-based approaches~\cite{meuleau2013solving}.  More recent work has
employed on-line sampling-based methods  to great effect, most notably
\cite{silver2010monte} and \cite{somani2013despot}, along with their descendants.
These methods explore only those belief states that can be reached
from the circumstances actually facing the robot, which helps increase the
scale of problems that can be effectively attacked. 

Still other techniques improve scalability further by treating what might be
termed ``intermediate'' formulations, imposing constraints derived from other
insights.  In terms of observations: for instance, the locally observable \mdps
of~\cite{merlin2020locally} consider observations derived from the
readings of realistic sensors. There, when something is sensed, it is sensed
well; when it is not observed, no data are obtained. In some ways this is akin to the
periodic observations we treat, as the challenge is sparseness rather than degradation or
corruption through random noise.  In terms of actions: the
options/macro-actions framework~\cite{sutton1999between} considers aggregate
actions, permitting a notion of hierarchical solution. They treat policies as
macro-actions but, for periodic observations, no observations occur between
the atomic actions, so we consider just simple sequences.  
Finally, 
unlike semi-Markov processes,
no different mathematical machinery will
be needed
for problems with periodic observations, 
other than some few complexities raised with regard to
discounting.

\vspace*{-0.2em}
\section{Preliminaries and basic definitions}
\label{sec:definitions}

We formally define the periodically state-observed Markov Decision Process (\psomdp) as follows.

\begin{definition}[\psomdp]
A \emph{periodically state-observed Markov Decision Process}
is a 5-tuple $\langle S, A, T, R, \checkin  \rangle$ where
\begin{defitems}
\item $S = \{s_0, s_1, \dots, s_{|S|}\}$ is the finite set of states; 
\item $A = \{a_0, a_1, \dots, a_{|A|}\}$ is the finite set of actions; 
\item $T: S\times A \times S \to [0,1]$ is the transition dynamics, or
transition model, describing the stochastic state transitions of the system,
assumed to be Markovian in the states, where $\forall t, P(s^{t+1} = s' | s^{t}
= s, a^{t} = a) = T(s',a,s)$;
\item $R: S\times A\to \Reals$ is the \gobble{reward }function which prescribes
that reward $R(s,a)$ is obtained for taking action $a$ in state~$s$;
\item $\checkin \in \Nats$ is the check-in period.
\end{defitems}
\end{definition}

The optimization objective is to maximize the expected discounted cumulative reward
\begin{equation}
U(\state^0)=\expectation{\sum_{t=0}^{\infty} \gamma^{\,t} R\left(s^t,a^t\right)}
\label{eqn:psomdp-discount}
\end{equation}
via selection of actions $a^{0}, a^{1}, a^{2}, \dots$.

The key difference with respect to standard \mdps is that, when an agent's planning
problem is modeled via a \psomdp, it must take actions at every time, but with
the current state being disclosed only every $\checkin$ steps: $t\in \{0, \checkin, 2\checkin, 3\checkin, \cdots, \}$.
Between check-ins, the agent cannot directly observe its own state, and it must maintain a belief over its state and plan based on this belief. 
Following standard notation, in what follows we write $U^{\opt}(\cdot)$ for the value function that gives the maximal expected discounted cumulative reward at each state.

The \psomdp problem can be cast both as a Markov Decision Process with composite (or macro) actions, and as a \pomdp with many uninformative observations.

\subsection{Equivalent \mdp formulation}

To rigorously define a solution concept for an \psomdp (i.e., to show the concept of a policy is appropriate), we
first need two definitions, which we shall re-use later too.

\begin{definition}[transition composition]
\label{defn:t-composition}
For some $\checkin \in \Nats$ and transitions $T: S\times A \times S \to [0,1]$ the
\emph{$\checkin$-composed transition model} is the function $T^{\checkin}: S\times A^\checkin \times S \to [0,1]$ defined as

\vspace*{-1ex}
{\small
\begin{equation}
T^{\checkin}\left(s',(a_0, \dots, a_{\checkin-1}), s\right) = \!\!\!\!\!\!\!\!\!\sum_{\substack{(s_0, \dots, s_{\checkin-1})\in S^\checkin\\ \text{where } s_0 = s\\ \text{ and } s_{\checkin-1} = s'}} \prod_{i=0}^{\checkin-1} T(s_{i+1},a_i,s_i).
\label{eq:def:composite-transitions}
\end{equation}
}
\end{definition}

The transition model $T$ describes the distribution of states reached after a
single step, conditioned on a single action being issued. By unfurling copies of $T$, the
$\checkin$-composed version, $T^\checkin$, describes the distribution of states reached after
$\checkin$~steps, now conditioned on a sequence of $\checkin$~actions; we  
will write $\action$ for such sequences.  By definition, $T = T^1$.

\begin{definition}[Reward Composition]
\label{defn:rew-compositions}
For some $\checkin \in \Nats$, discounting factor $\gamma \in [0,1)$, and reward
function $R: S\times A\to \Reals$ the \emph{$\checkin$-composed $\gamma$-discounted reward} is
the function $R^{\checkin,\gamma}: S\times A^\checkin \to\Reals$ defined as

\vspace*{-2ex}
{\small
\begin{align}
&R^{\checkin,\gamma}\left(s,\action \right) = R^{\checkin,\gamma}\left(s,(a_0, a_1, \dots a_{\checkin-1})\right) = \nonumber \\[-4pt]
&\hspace*{8ex}\sum_{d = 0}^{\checkin-1} \gamma^{\;d}\!\!\left( \sum_{\substack{(s_0, \dots, s_{d})\in S^{d+1}\\ \text{where } s_0 = s}}  R(s_{d}, a_{d}) \prod_{i=0}^{d-1} T(s_{i+1},a_i,s_i)\right).
\label{eq:def:composite-rewards}
\end{align}
}
\end{definition}

The $\checkin$-composed version of the reward function is analogous to the transition composition, but with the
additional complexity that the discount is incorporated as one runs along the
length of the sequence. Note that, by definition $R = R^{1,\gamma}$. In circumstances,
like this one here, where $\gamma$ plays no role it will be elided and we will write $R^{1}$
only.

The \checkin-composed transition model \eqref{eq:def:composite-transitions} and the \checkin-composed rewards \eqref{eq:def:composite-rewards} can be computed recursively; one can show that the resulting computation time grows exponentially with the check-in period \checkin as $O((|\states|\cdot|\actions|)^\checkin)$.

We are now in a position to define an \mdp equivalent to any \psomdp:

\begin{definition}[Composite Action Process]
\label{defn:composite-action}
Given \psomdp $\instance{M} = \langle S, A, T, R, \checkin  \rangle$,
its associated \emph{composite action decision process} is 
$\instance{M}_{\comp} = \langle S, A^\checkin, T^\checkin, R^{\checkin,\gamma}, 1  \rangle$.
\end{definition}

The composite action process is an \mdp 
because the Markov property is preserved when state sequences are gathered together,
indicating that it has
solution in the form of a mapping from states to $\checkin$-length sequences of
actions, \emph{viz.} a policy.  By ``solution'' here, we mean actions that yield an optimal cumulative
reward in expectation over the stochastic transition dynamics.
Since $\instance{M}$ and $\instance{M}_{\comp}$ are really identical problems
on the same Markov process, every \psomdp has a solution in the
form of a policy. The optimal state-action values, or Q-values, can be computed as

\vspace*{-2ex}
\begin{equation}
Q^\opt(\state,\action) = R^{\checkin,\gamma}(\state,\action) + \gamma^\checkin \sum_{\state'\in\states} T^\checkin(\state',\action,\state) \max_{\action'\in\actions^\checkin} Q^\opt(\state',\action')
\label{eq:mdp:qvalue},
\end{equation}
and the corresponding optimal policy can be computed as
\begin{equation}
\pi^\opt(\state) = \arg\max_{\action\in\actions^\checkin} Q^\opt(\state,\action).
\label{eq:mdp:policy}
\end{equation}

\subsection{Equivalent \pomdp formulation}
The \psomdp can also be cast as a \pomdp where the state $S_\pomdp = \{s_0, s_1, \dots, s_{|S|}\}\times\{[0,\dots, \checkin-1\}$ captures the \psomdp state and the time until the next check-in; the actions, transitions, and rewards are identical to the \psomdp's (with two minor exceptions: the transitions also update the time until the next check-in, and the rewards ignore the temporal portion of the state); and the observation function $O_\pomdp$ returns the current state at times corresponding to check-ins, and is uninformative otherwise.

\subsection{Discussion}

Note how, in the preceding, the \pomdp treatment is a poor fit for
the \psomdp sub-class of problems. We are required to inflate the state space
to account for the check-in period because the observations depend on the time
since the last check-in, but 
must be conditioned on state.
Also, the expressive freedom which the \pomdp does provide, a distribution in
$O_\pomdp$, can't be turned to advantage.

The composite action \mdp suffers from problems too. Its action space is exponential in the size of the \psomdp's; indeed, as \checkin grows, the possibility of obtaining any solution in this form looks increasingly implausible.  
Part of the problem is that standard 
\mdp solution 
techniques treat the action set as an
opaque collection. The fact that these particular actions are sequences of more atomic
actions suggests that it could be useful to consider interrelationships
between solutions with differing actions. 
This motivates the search for upper and lower bounds, which follows next.
However, to do so we find that it aids the intuition to adopt an
information-oriented interpretation.

\section{Upper Bounds}
\label{sec:ub}
In this section, we explore how additional information check-ins can provide
upper bounds on the value of \psomdp problems. First, we introduce an auxiliary
definition.

\begin{notation}[Action sequence subset]
Consider an action sequence $\action=(a_0,a_1,\ldots, a_{\checkin-1})$ of
length $\checkin$. We denote as $\action_{\ell:m}$ the subsequence $(a_\ell,
a_{\ell+1}, \ldots, a_{m-1})$ of length $m-\ell$. %
\end{notation}

\subsection{Bonus and extra check-ins}
\label{sec:ub:bonus}

We start by assessing the value of receiving supererogatory 
check-ins in addition to the periodic check-ins that occur with period \checkin. We distinguish two situations: \emph{announced extra check-ins}, when the availability of a future additional check-in is known in advance, and \emph{unannounced bonus check-ins}, where the occurrence of the check-in is not anticipated.

\begin{definition}[Unannounced bonus check-in]
Consider an agent following an optimal \psomdp policy. Suppose that $\tau$ time
steps after the last check-in (with $\tau<\checkin$), the agent receives an
unanticipated
check-in, which reveals its state; the agent can use this
newly-disclosed bonus information to optimize the expected discounted reward.
\end{definition}

\begin{definition}[Announced extra check-in]
Suppose an agent is following an optimal \psomdp policy. At the time of a
check-in, the agent is informed that it will receive an extra check-in after
$\tau<\checkin$ time steps, \emph{in addition} to the regularly scheduled
check-ins. The agent can use this newly-disclosed information to optimize the
expected discounted reward.
\end{definition}

The next two lemmas show that, perhaps unsurprisingly, both unannounced and announced check-ins do not decrease the expected reward, and announced check-ins never result in a lower reward compared to unannounced ones.

\begin{lemma}[Unannounced bonus check-ins bound \psomdp values from above]
\label{lemma:bonus:unannounced}
Consider a \psomdp with an unannounced bonus check-in $\tau$ time steps after a regular check-in. The optimal policy that uses the information provided by the unannounced bonus check-in has an expected discounted reward no lower than the original \psomdp policy.
\end{lemma}
\begin{proof}
The optimal reward for an agent in state \state, $\tau$ time steps after the last check-in, when the bonus check-in occurs, can be computed as

\vspace*{-3ex}
{\small
\begin{align}
U^{\opt}_{U,[\tau]}(s)=&\max_{\action_{\tau:\checkin}\in\actions^{\checkin-\tau}} \Bigg( R^{\checkin-\tau,\gamma}(\state,\action_{\tau:\checkin}) + \nonumber \\[-8pt]
&\qquad\qquad\gamma^{\checkin-\tau}\sum_{\state'\in\states} T^{\checkin-\tau}(s',\action_{\tau:\checkin},s) U^{\opt}(\state') \Bigg).
\label{eq:ub:checkin:unannounced}
\end{align}
}
In contrast, in absence of the bonus check-in, the agent executes the tail of the action $\hat \action=\pi^{\opt}_\psomdp(s)$ computed at the last check-in, $\tau$ time steps before. The reward that results is just

\vspace*{-1ex}
{\small
\begin{equation}
R^{\checkin-\tau,\gamma}(\state,\hat \action_{\tau:\checkin}) + \gamma^{\checkin-\tau}\sum_{\state'\in\states} T^{\checkin-\tau}(s',\hat \action_{\tau:\checkin},s) U^{\opt}(\state').
\end{equation}
}
\vspace*{-1ex}

Since $\hat \action_{\tau:\checkin}\in\actions^{\checkin-\tau}$, $\hat \action_{\tau:\checkin}$ is an admissible solution to \eqref{eq:ub:checkin:unannounced}, and hence, the claim follows.  
\end{proof}

Lemma~\ref{lemma:bonus:unannounced} focuses on state values. In contrast, for announced extra check-ins, we start by providing a bound on Q-values as follows.

\begin{lemma}[Announced extra check-ins bound \psomdp Q-values from above]
\label{lemma:bonus:announced:qvalues}
Consider a \psomdp with an announced extra check-in $\tau$ time steps after a regular check-in. Denote the Q-value of a state-action pair $(\state, \action)$ under the optimal \psomdp policy that ignores the additional check-in as $Q^\opt(\state,\action)$;
and denote the optimal Q-value of the state-action pair $(\state, \action_{0:\tau})$ that uses the extra check-in information as $Q_1^\opt(\state,\action_{0:\tau})$ (where the subscript refers to the fact that a single extra check-in is provided).
Then, $Q_1^\opt(\state,\action_{0:\tau}) \geq Q^\opt(\state,\action)$ for all $\state\in\states, \action\in\actions^\checkin$.
\end{lemma}
\begin{proof}
The optimal Q-value of a state-action pair under a policy that uses the extra check-in information is

{\small
\begin{align}
&Q_1^\opt(\state,\action_{0:\tau}) = R^{\tau,\gamma}(\state,\action_{0:\tau}) + \gamma^\tau \sum_{\state'\in\states} T^{\tau}(s',\action_{0:\tau},s) U^{\opt}_{U,[\tau]}(\state')  = \nonumber \\
&\quad R^{\tau,\gamma}(s,\action_{0:\tau}) + \gamma^\tau \sum_{\state'\in\states} T^{\tau}(\state',\action_{0:\tau},s) \;\;\times \nonumber \\[-10pt]
&\!\!\max_{\action_{\tau:\checkin}\in\actions^{\checkin-\tau}}\Bigg[  R^{\checkin-\tau,\gamma}(\state',\action_{\tau:\checkin}) \;\;+\nonumber  \\[-10pt]
&\hspace*{5.85ex}\gamma^{\checkin-\tau} \sum_{\state''\in\states} T^{\checkin-\tau}(s'',\action_{\tau:\checkin},s') \bigg(\max_{\action''\in\actions^\checkin}Q^{\opt}(\state'',\action'')\bigg)%
\Bigg].\label{eq:upperbound:qvalue:ub}
\end{align}
}

Recall that Equation \eqref{eq:mdp:qvalue} captures the Q-value $Q^\opt(\state,\action)$ of a state-action pair under the optimal policy that ignores the extra check-in.
Using Definition \ref{defn:rew-compositions}, rewrite the reward $R^{\checkin,\gamma}(\state, \action)$ in \eqref{eq:mdp:qvalue} as

{\small
\begin{align}
R^{\checkin,\gamma}(\state, \action)\! =\! R^{\tau,\gamma}(\state,\action_{0:\tau})+ \gamma^{\tau}\! \sum_{\state'\in\states} T^{\tau}\!(\state',\action_{0:\tau},\state) R^{\checkin-\tau,\gamma}(\state',\action_{\tau:\checkin}), \label{eq:upperbound:reward}
\end{align}
}
and use Definition \ref{defn:t-composition} to rewrite the reward-to-go as
{\small
\begin{align}
& \sum_{\state''\in \states} T{^\checkin}(\state'',\action,\state) \max_{\action''\in\actions^\checkin} Q^\opt(\state'',\action'')= \label{eq:upperbound:cost-to-go}\\[-6pt]
& \quad \sum_{\state'\in \states} T^\tau(\state',\action_{0:\tau},\state) \sum_{\state''\in\states} T^{\checkin-\tau}(\state'',\action_{\tau:\checkin}, \state') 
\bigg(\max_{\action''\in\actions^\checkin} Q^\opt(\state'',\action'')\bigg). \nonumber
\end{align}
}

Replacing \eqref{eq:upperbound:reward} and \eqref{eq:upperbound:cost-to-go} in \eqref{eq:mdp:qvalue}, we obtain
{\small
\begin{align}
Q^{\opt}(\state,\action) =&\;R^{\tau,\gamma}(\state,\action_{0:\tau})+\gamma^{\tau} \sum_{\state'\in\states} T^{\tau}(\state',\action_{0:\tau},\state) R^{\checkin-\tau,\gamma}(\state',\action_{\tau:\checkin}) + \nonumber \\[-6pt]
&\hspace*{-7ex} \gamma^{\tau}\gamma^{\checkin-\tau} \sum_{\state'\in \states} T^\tau(\state',\action_{0:\tau},\state) \sum_{\state''\in\states} T^{\checkin-\tau}(\state'',\action_{\tau:\checkin}, \state')  \bigg(\max_{\action''\in\actions^\checkin} Q^\opt(\state'',\action'')%
\bigg) \nonumber\\[12pt]
=&\;R^{\tau,\gamma}(\state,\action_{0:\tau}) + \gamma^{\tau} \sum_{\state'\in\states}  T^{\tau}(\state',\action_{0:\tau},\state) \;\;\times \label{eq:upperbound:qvalue:expanded} \\[-6pt]
& \hspace*{-9ex}\Bigg( R^{\checkin-\tau,\gamma}(\state',\action_{\tau:\checkin}) + \gamma^{\checkin-\tau}\!\sum_{\state''\in\states}\!\!T^{\checkin-\tau}(\state'',\action_{\tau:\checkin}, \state') \bigg(\max_{\action''\in\actions^\checkin}  Q^\opt(\state'',\action'')\bigg)
\Bigg).  \nonumber
\end{align}
}
Comparing \eqref{eq:upperbound:qvalue:expanded} with \eqref{eq:upperbound:qvalue:ub}, one can see that former is an admissible solution to the maximization problem in the latter: therefore, $Q_1^\opt(\state,\action_{0:\tau})\geq Q^{\opt}(\state,\action)$.
\end{proof}

We then use Lemma~\ref{lemma:bonus:announced:qvalues} to provide a bound on state values.

\begin{lemma}[Announced extra check-ins bound \psomdp state values from above]
Consider a \psomdp with an announced extra check-in $\tau$ time steps after a regular check-in. The optimal policy that uses the information provided by the announced extra check-in has an expected discounted reward no lower than  the original \psomdp policy, and no lower than the reward from an unannounced bonus check-in at time $\tau$.
\label{lemma:bonus:announced}
\end{lemma}
\begin{proof}
The optimal reward for an agent in state $s$ when an announced extra check-in is revealed is
\begin{equation}
U^{\opt}_{[\tau\,]}(\state)=\max_{\action_{0:\tau}\in\actions^{\tau}} Q_1(\state,\action_{0:\tau}). \label{eq:upperbound:announced:U}
\end{equation}
In contrast, the expected discounted reward if no extra check-ins are available can be written as
$U^{\opt}(\state)= 
\max_{\action\in\actions^\checkin}Q(\state,\action).$
Lemma~\ref{lemma:bonus:announced:qvalues} shows that $Q_1(\state,\action_{0:\tau})\geq Q(\state,\action), \forall \state\in\states, \action\in\actions^\checkin$; the claim follows.

We also show that the reward \eqref{eq:upperbound:qvalue:ub} is no lower than the corresponding reward for an unannounced bonus.
The reward in state $s$ for an agent that will receive a bonus check-in after $\tau$ time steps (but does not know it yet) is
\begin{align*}
R^{\tau,\gamma}(s,\hat \action_{0:\tau}) + \gamma^\tau \sum_{\state'\in\states} T^{\tau}(s',\hat \action_{0:\tau},s) U^{\opt}_{U,\tau}(\state') = Q_1^\opt(\state,\hat \action_{0:\tau})
\end{align*}
where $\hat \action=\pi_\psomdp^\opt(\state)$ follows the optimal policy in absence of check-ins. The action prefix $\hat \action_{0:\tau}$ is an admissible solution to the maximization problem in \eqref{eq:upperbound:announced:U}; the claim follows.
\end{proof}

Next, we consider the effect of adding announced extra check-ins after \emph{every} regular check-in.

First, we provide an auxiliary definition.
\begin{definition}[\psomdp with additional check-ins]
\label{defn:psomdp-additional-checkins}
Consider a \psomdp \instance{M} with check-in period $\checkin$.
We define a \psomdp with additional check-ins $\instance{\hat M}^{\tau}$  as a modification of  \psomdp \instance{M} where, 
after an action is taken, an \emph{additional} check-in occurs after $\tau<\checkin$ steps. 
That is, a policy for $\instance{\hat M}^{\tau}$ specifies an action of length
$\tau$ (to be taken if the previous action was of length $\checkin - \tau)$,
and an action of length  $\checkin - \tau$ (to be taken if the previous action
was of length $\tau$) for every state. We denote as $Q(s,\action_{0:\tau})$ the Q-value
of the state-action pair $s,\action_{0:\tau}$, i.e. the set of values that
satisfy

{\small
\begin{align*}
Q^\opt(s,\action_{0:\tau}) = &\;R^{\tau,\gamma}(s, \action_{0:\tau}) + \gamma^\tau\!\sum_{s'\in S}\!T^{\tau}(s', \action_{0:\tau}, s)\!\!\max_{\action_{\tau:\checkin}' \in A^{\checkin - \tau}}\!Q^{\opt}(s',\action_{\tau:\checkin}'), \\[4pt]
Q^\opt(s,\action_{\tau:\checkin}) = &\;R^{\checkin-\tau,\gamma}(s, \action_{\tau:\checkin}) + \\
& \hspace*{10ex}\gamma^{\checkin-\tau} \sum_{s'\in S} T^{\checkin-\tau}(s', \action_{\tau:\checkin}, s)\max_{\action_{0:\tau}' \in A^{\tau}}\!Q^{\opt}(s',\action_{0:\tau}').
\end{align*} 
}
\end{definition}

\begin{theorem}[extra check-ins  bound \psomdp Q-values from above]
\label{thm:superset-checkins-upperbound}

Suppose $\instance{M}$ is  a \psomdp with check-in period $\checkin$, and denote the associated optimal Q-values as $Q^{\opt}(\state,\action)$. Also consider a modified \psomdp $\instance{\hat M}^{\tau}$  with additional check-ins at $\tau$, as per Definition~\ref{defn:psomdp-additional-checkins}, 
and denote the associated state values as $\hat Q^{\opt}(\state,\action_{0:\tau})$.
Then, $Q^{\opt}(\state,\action)\leq \hat Q^{\opt}(\state,\action_{0:\tau})$ for all $\state\in\states,\action\in\actions^{\checkin}$.
\end{theorem}
\begin{proof}
The proof is by induction on the number of extra check-ins. We define as $\hat Q_{\ell}^\opt(\state,\action)$ the optimal Q-value when $\ell$ consecutive extra check-ins are provided, with $\hat Q_0(\state,\action_{0:\tau})=Q(\state,\action)$, and $\hat Q_1(\state,\action_{0:\tau})$ defined in Equation \eqref{eq:upperbound:qvalue:ub}. 
We show that, for all $\ell\in\mathbb{N}$, $\hat Q_{\ell+1}^\opt(\state,\action_{0:\tau})\geq \hat Q_{\ell}^\opt(\state,\action_{0:\tau})$, which implies that 
$\hat Q_{\ell+1}^\opt(\state,\action_{0:\tau})\geq \hat Q_{0}^\opt(\state,\action_{0:\tau})=Q^\opt(\state,\action)$.
 \\
\indent\emph{Base case:} One extra check-in, is Lemma~\ref{lemma:bonus:announced:qvalues}. 
\\
\indent\emph{Inductive case:} \gobble{Next, we show that, if receiving bonus check-ins for $\ell$ periods improves the state value, then receiving bonus check-ins for $\ell+1$ periods also does.}
We wish to show that, if $\hat Q_{\ell-1}^{\opt}(\state,\action_{0:\tau})\leq \hat Q_{\ell}^{\opt}(\state,\action_{0:\tau})$, then $\hat Q_{\ell}^{\opt}(\state,\action_{0:\tau})\leq Q_{\ell+1}^{\opt}(\state,\action_{0:\tau})$.
We can express $\hat Q_{\ell+1}^{\opt}(\state,\action_{0:\tau})$ as

\vspace*{-2ex}
{\small
\begin{align}
&\hat Q_{\ell+1}^{\opt}(\state,\action_{0:\tau}) =  R^{\tau,\gamma}(\state, \action_{0:\tau}) + \gamma^{\tau} \sum_{\state'\in\states} T^\tau(\state',\action_{0:\tau},\state) \times \nonumber  \\[-4pt]
&\hspace*{3ex} \max_{\action'_{\tau:\checkin}\in\actions^{\checkin-\tau}} \Bigg[ R^{\checkin-\tau,\gamma}(\state',\action_{\tau:\checkin})\; +  \label{eq:extra-checkins:lp} \\[-4pt]
&\hspace*{7ex}
 \gamma^{\checkin-\tau} \sum_{\state''\in\states} T^{\checkin-\tau}(\state'',\action_{\tau:\checkin}, \state')\times
\bigg(\max_{\action''_{0:\tau}\in\actions^{\tau}}Q_\ell^{\opt} (\state'',\action''_{0:\tau})\bigg) \Bigg].\nonumber\end{align}
}
Analogously, we can express $\hat Q_{\ell}^{\opt}(\state,\action_{o:\tau})$ as
{\small
\begin{align}
&\hat Q_{\ell}^{\opt}(\state,\action_{0:\tau}) = R^{\tau,\gamma}(\state, \action_{0:\tau}) + \gamma^{\tau} \sum_{\state'\in\states} T(\state',\action_{0:\tau},\state) \times \nonumber  \\[-4pt]
&\hspace*{2ex} \max_{\action'_{\tau:\checkin}\in\actions^{\checkin-\tau}} \Bigg[ R^{\checkin-\tau,\gamma}(\state',\action_{\tau:\checkin})\; + \label{eq:extra-checkins:l}\\[-4pt]
&\hspace*{5ex}
 \gamma^{\checkin-\tau} \sum_{\state''\in\states} T^{\checkin-\tau}(\state'',\action_{\tau:\checkin}, \state')\times
\bigg(\max_{\action''_{0:\tau}\in\actions^{\tau}}Q_{\ell-1}^{\opt} (\state'',\action''_{0:\tau})\bigg) \Bigg]. \nonumber
\end{align}
}
The expressions for $\hat Q_{\ell+1}^{\opt}$ and $\hat Q_{\ell}^{\opt}$ present the same
nested maximization problems and have identical arguments except for
$\hat Q_{\ell}^{\opt}$ in \eqref{eq:extra-checkins:lp} and $\hat Q_{\ell-1}^{\opt}$ in
\eqref{eq:extra-checkins:l}; by the inductive assumption, $\hat Q_{\ell}^{\opt}\geq
\hat Q_{\ell-1}^{\opt}$, therefore $\hat Q_{\ell+1}^{\opt}\geq \hat Q_{\ell}^{\opt}$ and the claim
follows.
\end{proof}

The difference in expressions for values when 
check-ins are unannounced (Lemma~\ref{lemma:bonus:unannounced}) versus 
announced (Lemma~\ref{lemma:bonus:announced}), and indeed the
numerical difference in their respective value functions, expresses the value
of knowing beforehand that a check-in will occur. 
Like most information, this has value; 
but notice that announcements are second-order 
statements: they are statements \emph{about} subsequent disclosures.

\subsection{More frequent check-ins are not always beneficial}
\label{sec:ub:check-in-counter-example}

Next, in a perhaps unintuitive result, we show that increasing the frequency of check-ins is \emph{not} guaranteed to bound state values from above. To show this, we provide a counterexample in Figure~\ref{fig:ub:counterexample}. In it, an agent is tasked with navigating a simple grid world with multiple ranks of obstacles (shown in purple) set three steps away from each other; the goal is to reach the cyan cell on the right. The agent's control is imprecise: when commanded to drive in a given direction, the agent also drifts to the left or to the right of the desired direction with 5\% probability each. 
\begin{figure}[h]
\jvspace{-1em}
\centering
\begin{subfigure}[b]{.45\linewidth}
\centering
\begin{tikzpicture}
\node at (1.35,0) {\includegraphics[scale=0.175,angle=-90]{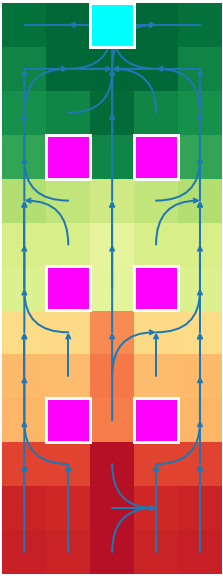}};
\node at (0,0) {\scriptsize\textcolor{white}{$\star$}};
\end{tikzpicture}
\caption{$\checkin=2$.}
\label{fig:ub:counterexample:2}
\end{subfigure}
\begin{subfigure}[b]{.45\linewidth}
\centering
\begin{tikzpicture}
\node at (0,0) {\includegraphics[scale=0.175,trim=0 0 8.2cm 0cm, clip,angle=-90]{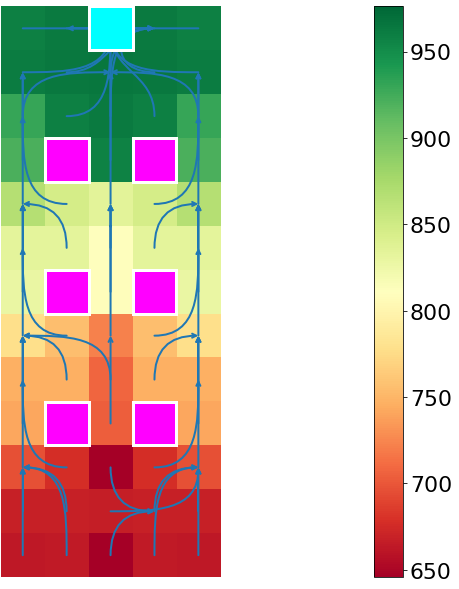}};
\node at (-1.35,0) {\scriptsize\textcolor{white}{$\star$}};
\end{tikzpicture}
\caption{$\checkin=3$.}
\label{fig:ub:counterexample:3}
\end{subfigure}
\includegraphics[scale=0.095,trim=15cm -2.6cm 0cm 0cm, clip]{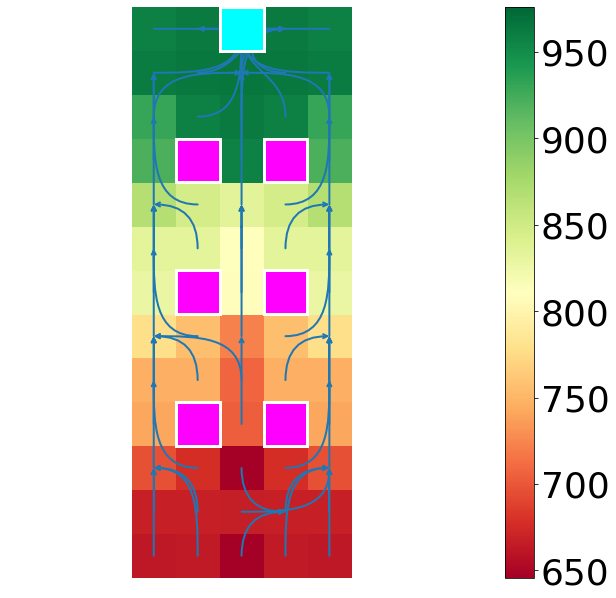}
\caption{State values and policy for a grid-world \psomdp. Increasing the check-in frequency can be detrimental to state values. The expected reward for the starred location in (b) is higher than that same location in (a).}
\label{fig:ub:counterexample}
\jvspace{-1em}
\end{figure}

We compare two \psomdps with $\checkin=2$ (Figure~\ref{fig:ub:counterexample:2}) and $\checkin=3$ (Figure~\ref{fig:ub:counterexample:3}), respectively.
When the time between check-ins is $\checkin=3$, corresponding to the
stride between obstacles, the state values (and therefore the Q-values) for
many cells farthest from the goal are \emph{higher} compared to the case with
more frequent check-ins. (The star helps indicate one especially clear example.)
Intuitively, more infrequent check-ins provide information at times that are well-attuned with the environment, allowing the agent to recognize its position just before traversing each rank of obstacles.

\vspace*{-4ex}
\subsection{Extended Q-values}
\vspace*{-2ex}
\label{sec:ub:extended-Q-values}

We generalize the upper bounds identified in Section \ref{sec:ub:bonus} by
introducing the notion of \emph{extended Q-values}.
Intuitively, extended Q-values generalize the notion of announced extra check-ins in two ways: they allow extra check-ins to occur periodically, as opposed to once, and also allow multiple extra check-ins to occur between pairs of regularly-scheduled ones. By changing the times at which extra check-ins occur, extended Q-values provide a way of generating  families of
upper bounds for a given \psomdp problem.
Adding additional extra check-ins to a given instance results in lower computational complexity, at the price of a looser upper bound; however we note that, as discussed in the example above, simply increasing the frequency of extra check-ins does not necessarily result in looser upper bounds.

To start with, we consider $\checkin$ separate state value functions, $Q^{\opt}_{[m]}: S\times
\actions \to \Reals, m \in \{0,\dots,\checkin-1\}$, each
corresponding to the addition of an unannounced bonus check-in at offset
$\checkin-m$ from the regular check-ins, defined
as follows: %

\vspace*{-2ex}
{\small
\begin{subequations}
{\small
\begin{align}
Q^{\opt}_{[0]}(\state,\action)&=Q^{\opt}(\state,\action),\\[6pt]
Q^{\opt}_{[m]}(s,\action) & = R^{m,\gamma}(s, \action_{0:m})\;+\;\gamma^m \sum_{s'\in S} T^{m}(s', \action_{0:m}, s)\;\times\label{eq:bounds:ub:qvalues:base:m} \\[-6pt] 
& \hspace*{30ex}\max_{\action' \in A^\checkin}Q^{\opt}(s',\action').\nonumber
\end{align}
}
\label{eq:bounds:ub:qvalues:base}
\end{subequations}
}
Note that Equation \eqref{eq:bounds:ub:qvalues:base:m} is identical to the
argument of the minimizer in Equation \eqref{eq:ub:checkin:unannounced}.  
Next, we rewrite \eqref{eq:bounds:ub:qvalues:base}, by admitting the
possibility of receiving, additionally, an announced extra check-in, but selecting
its placement so that it results in the
\emph{smallest} Q-value:\pagebreak %

\vspace*{-5.5ex}
\begin{subequations}
{\small
\begin{align}
Q^{\opt}_{[0]}(\state,\action)  = &\min_{\ell \in \{1,\dots,\checkin\}} \bigg(R^{\ell,\gamma}(s, \action_{0:\ell}) + \nonumber \\[-6pt]
& \hspace*{6ex}\gamma^\ell \sum_{s'\in S} T^{\ell}(s', \action_{0:\ell}, s) \max_{\action'\in\actions^\checkin}Q^{\opt}_{[\checkin-{\ell}]}(s',\action') \bigg),\\[2pt]
Q^{\opt}_{[m]}(\state,\action) = & \min_{\ell \in \{1,\dots,m\}} \bigg(R^{\ell,\gamma}(s, \action_{0:\ell}) + \nonumber \\[-6pt]
&\hspace*{6ex}\gamma^\ell \sum_{s'\in S} T^{\ell} (s', \action_{0:\ell}, s)\max_{\action'\in\actions^\checkin} Q^{\opt}_{[m-\ell]}(s',\action') \bigg).
\end{align}
}
\label{eq:bounds:ub:qvalues:ext}
\end{subequations}
As shown in Theorem~\ref{thm:superset-checkins-upperbound}, the minimum state value is achieved when $\ell=m$, and no extra check-ins are provided, thus Equations \eqref{eq:bounds:ub:qvalues:ext} are \emph{equivalent} to \eqref{eq:bounds:ub:qvalues:base}.
But, importantly, Equations \eqref{eq:bounds:ub:qvalues:ext} can be manipulated
to create families of upper bounds by removing selected entries from the
argument of the $\min$ operator. 
We pick extra check-in times
so each can form some subset 
summing to $\checkin$, i.e., 
$B\subset \{1,\dots, \checkin\}$ 
satisfying the following:

\vspace*{-1.6ex}
{\small
\[
\forall \ell\in B,\; \exists \hat B_\ell\subseteq B \text{ such that } \bigg(\ell+\sum_{\hat \ell \in \hat B_\ell} \hat \ell\bigg) = \checkin.
\]
}

Intuitively, the property ensures that the extra 
check-in times can be ``composed'' to achieve a stride of \checkin. 
In practice this is easily achieved, e.g., by selecting $B$ to contain divisors of
\checkin, or by ensuring that $(\ell\in B)\iff
((\checkin-\ell)\in B)$.

Next, we modify \eqref{eq:bounds:ub:qvalues:ext} as follows: 

\vspace*{-2.0ex}
\begin{subequations}
{\small
\begin{align}
Q^\opt_{[0]}(\state,\action) \leq \overline Q^\opt_{[0]}(\state,\action) =&\min_{\ell \in B} \Bigg(R^{\ell,\gamma}(s, \action_{0:\ell}) + \label{eq:bounds:ub:qvalues:mod:0} \\[-8pt]
& \gamma^\ell \sum_{s'\in S} T^{\ell}(s', \action_{0:\ell}, s) \max_{\action'\in\actions^\checkin}\overline Q^{\opt}_{[\checkin-{\ell}]}(s',\action') \Bigg), \nonumber\\[-2pt]
Q^{\opt}_{[m]}(\state,\action) \leq \overline Q^{\opt}_{[m]}(\state,\action)  =&\min_{\ell \in B, \ell\leq m} \Bigg(R^{\ell,\gamma}(s, \action_{0:\ell}) + \label{eq:bounds:ub:qvalues:mod:m} \\[-8pt]
&\gamma^\ell \sum_{s'\in S} T^{\ell} (s', \action_{0:\ell}, s)\max_{\action'\in\actions^\checkin} \overline Q^{\opt}_{[m-\ell]}(s',\action') \Bigg), \nonumber %
\end{align}
}
\label{eq:bounds:ub:qvalues:mod}
\end{subequations}

\vspace*{-1.0ex}

\noindent where the \emph{in}equality here (unlike the equality in \eqref{eq:bounds:ub:qvalues:ext}) derives from the fact that the minimum is taken over only a subset of all possible check-in times.
In \eqref{eq:bounds:ub:qvalues:mod:0} the ``less than'' requirement of the minimizer is omitted because 
$Q^\opt_{[0]}$ and $\overline Q^\opt_{[0]}$ correspond to $m = \checkin$, and $\ell\in B$ is less than $\checkin$ by definition.

Evaluating the upper bounds in \eqref{eq:bounds:ub:qvalues:mod} requires solving a set of up to $\checkin$ coupled \mdps with actions of length $\ell\in B$.
A standard procedure for solving \mdps is Value Iteration~\cite{bertsekas19reinforcement}, which has time
complexity $O(|A||S|^2)$ per iteration\,\cite{russell09artificial}. Hence, the complexity of solving a \psomdp as an \mdp (following Definition \ref{defn:composite-action}) grows exponentially with the time between check-ins \checkin as $|A|^{\checkin}$. Actions in \eqref{eq:bounds:ub:qvalues:mod} have length $\ell<\checkin$; therefore, solving \eqref{eq:bounds:ub:qvalues:mod} for carefully-selected values of $\ell\in B$ can provide upper bounds at much lower cost compared to solving the original \psomdp problem as an \mdp.

\subsection{Selecting the set of check-in times $B$}

Selection of the set $B$ is critical to achieve a good balance between computational complexity and tightness of the bounds $\overline Q^\opt(\state,\action_{0:\overline\ell})$. In this section, we show that selecting $B$ to contain a single divisor of \checkin, i.e. $B=\{\ell\}$ for $\checkin=n\ell, n\in\mathbb{N}$, results in significant computational savings.

Equation \eqref{eq:bounds:ub:qvalues:mod} becomes

\vspace*{-1ex}
\begin{subequations}
{\small
\begin{align}
\overline Q^\opt_{[0]}(\state,\action) =&\;R^{\ell,\gamma}(s, \action_{0:\ell}) +  \gamma^\ell \sum_{s'\in S} T^\ell(s', \action_{0:\ell}, s) \max_{\action'\in\actions^\checkin}\overline Q^{\opt}_{[\checkin-{\ell}]}(s',\action') , \label{eq:bounds:ub:qvalues:BL:0} %
\end{align}
}
\label{eq:bounds:ub:qvalues:B1}
\end{subequations}

\addtocounter{equation}{-1}
\vspace*{-5.2ex}
\begin{subequations}
\stepcounter{equation}
{\small
\begin{align}
\overline Q^{\opt}_{[m]}(\state,\action)  =&\;R^{\ell,\gamma}(s, \action_{0:\ell}) + \gamma^\ell \sum_{s'\in S} T^\ell (s', \action_{0:\ell}, s)\max_{\action'\in\actions^\checkin} \overline Q^{\opt}_{[m-\ell]}(s',\action') ,\label{eq:bounds:ub:qvalues:B1:m}  \\[-6pt]
&\hspace*{6ex}\text{for those } m \in \{\ell, 2\cdot\ell,\dots,\checkin-\ell\}.\notag
\end{align}
}
\end{subequations}
\vspace*{-1ex}

By writing out the Equations~\eqref{eq:bounds:ub:qvalues:B1} for $\ell$, and
$2\cdot\ell$, and then $3\cdot\ell$, and so on, we observe two things.
Firstly, each is only concerned with the choice of action sequences of length exactly~$\ell$.  
Secondly, they are all, actually, posing precisely the same
optimization problem.  The separate degrees-of-freedom offered by having
multiple functions, like both $\overline Q^\opt_{[\ell]}(\state,\action)$ and
$\overline Q^\opt_{[2\cdot\ell]}(\state,\action)$, is unnecessary at the
optimum.  The $\frac{\checkin}{\ell}$ copies of the function are redundant
because, in all these cases, the agent begins at a known state, and solving
over a sequence of $\ell$ steps, arrives then at a state which will be
observed.  Therefore, we drop the bracketed subscript and the admissible solution with these extra check-ins has
just:

\vspace*{-1ex}
{\small
\begin{align}
\overline Q^\opt(\state,\action_{0:\ell}) =&R^{\ell,\gamma}(s, \action_{0:\ell}) +  \gamma^\ell \sum_{s'\in S} T^\ell(s', \action_{0:\ell}, s) \max_{\action'\in\actions^\ell} \overline Q^\opt(\state,\action_{0:\ell}'). \label{eq:bounds:ub:qvalues:BL:S}
\end{align}
}
(We have been especially explicit in our notation to emphasize that all of these actions have length only $\ell$.) 
Thus, this is equivalent to a \mdp with $|S|$ states and $|A|^\ell$ actions.

\smallskip

A special case of the selection above is $B=\{1\}$, corresponding to an \emph{omniscient relaxation} where the policy has access to state information at all time steps.

\section{Lower Bounds}
\label{sec:lb}

Next, we turn our attention to establishing lower bounds that provide suboptimal, feasible policies for the \psomdp that are computationally efficient to compute.

To achieve these lower bound, we reduce the action space and only search through \emph{action prefixes} of length $\tau<\checkin$.
Consider an arbitrary, fixed action suffix $\actionsuffix = (a_{\tau+1}, a_{\tau+2}, \ldots, a_{\checkin})$ of length $\checkin-\tau$. We consider the set of all actions with suffix $\actionsuffix$, that is,
\[
\underline \actions_\tau^\actionsuffix = \{\action\in\actions^\checkin \mid \action_{\tau:\checkin}=\actionsuffix\}=\{(\action \actionsuffix) \mid \action \in \actions^\tau\}.
\]
In order to achieve a lower bound, we solve a restricted \psomdp where the action space is limited to $\underline \actions_\tau^\actionsuffix$. The corresponding Q-values can be computed as

{\small
\begin{equation}
\underline Q^\opt(\state,\action) = R^{\checkin,\gamma}(\state,\action) + \gamma^\checkin \sum_{\state'\in\states} T^\checkin(\state',\action,\state) \max_{\action'\in\underline \actions_\tau^\actionsuffix} \underline Q^\opt(\state',\action'),
\label{eq:lowerbound:qvalue}
\end{equation}
}
for all $\action\in\underline\actions_\tau^\actionsuffix$.

Solving Equation \eqref{eq:lowerbound:qvalue} through value iteration incurs a computational complexity of $O(|S|^2|A|^{\ell})$, which is significantly smaller than the complexity of solving the full \psomdp as a \mdp, i.e.,  $O(|S|^2|A|^{\checkin})$.

The following lemma shows that $\underline Q^\opt(\state,\action)$ is indeed a lower bound on $Q^\opt(\state,\action)$.
\begin{lemma}[Restricting the action set lower-bounds the Q-values of the selected actions]
\label{lemma:lb:base}
Consider a \psomdp  $\langle S, A, T, R, \checkin  \rangle$ with Q-values $Q^\opt(\state,\action)$. Also consider the restriction of the \psomdp to action sequences $\underline \actions_\tau^\actionsuffix$, with Q-values $\underline Q^\opt(\state,\action)$. Then, $\underline Q^\opt(\state,\action)\leq Q^\opt(\state,\action) \forall \state\in\states, \action\in\underline\actions_\tau^\actionsuffix$.
\end{lemma}
\begin{proof}[Proof Sketch]
The claim follows from the observation that \eqref{eq:lowerbound:qvalue} and \eqref{eq:mdp:qvalue} have the same structure, and the maximization in \eqref{eq:lowerbound:qvalue} is on a smaller set of actions.
\end{proof}

The bound in Lemma~\ref{lemma:lb:base} only applies to the Q-values corresponding to actions with suffix $\actionsuffix$. Next, we extend the bound to all actions with a given prefix, and to state values.

\begin{theorem}[Restricting the action set lower-bounds the Q-values of actions sharing the same prefix]
\label{thm:lb:ext}
Consider a \psomdp  $\langle S, A, T, R, \checkin  \rangle$ with Q-values $Q^\opt(\state,\action)$. Also consider the restriction of the \psomdp to action sequences $\underline \actions_\tau^\actionsuffix$ with suffix $\actionsuffix$, with Q-values $\underline Q^\opt(\state,\action)$. Define $Q^\opt(\state,\action_{0:\tau})$ as
\begin{equation}
Q^\opt(\state,\action_{0:\tau})=\max_{\substack{\hat\action\in\actions^\checkin\\[1pt]\text{\rm with }\\[-1pt] \hat\action_{0:\tau}=\action_{0:\tau}}} Q^\opt(\state,\hat\action).
\label{eq:lb:prefix:ext}
\end{equation}

Then
$Q^\opt(\state,\action_{0:\tau})\geq \underline Q^\opt(\state,(\action_{0:\tau}\actionsuffix))$.
\end{theorem}
\begin{proof}
According to Lemma~\ref{lemma:lb:base}, $\underline Q^\opt(\state,\action) \leq Q^\opt(\state,\action), \forall \state\in\states,
\action\in\underline\actions_\tau^\actionsuffix$.  In particular, $\forall \state\in\states$,
$\underline Q^\opt(\state,(\action_{0:\tau}\actionsuffix)) \leq Q^\opt(\state,(\action_{0:\tau}\actionsuffix))
\leq \max_{{\hat\action\in\actions^\checkin\text{ with } \hat\action_{0:\tau}= \action_{0:\tau}}} Q^\opt(\state,\hat\action) = Q^\opt(\state,\action_{0:\tau}),$
the last inequality follows because $(\action_{0:\tau}\actionsuffix) \in\actions^\checkin$ with $(\action_{0:\tau}\actionsuffix)_{0:\tau} = \action_{0:\tau}$.
\end{proof}

\begin{lemma}[Restricting the set of admissible actions lower-bounds the state values]
\label{lemma:lb:states}
Consider a \psomdp  $\langle S, A, T, R, \checkin  \rangle$ with Q-values $Q^\opt(\state,\action)$ and state values $U^\opt(\state)$. Also consider the restriction of the \psomdp to action sequences $\underline \actions_\tau^\actionsuffix$, with Q-values $\underline Q^\opt(\state,\action)$. Then, 
\begin{align}
U^\opt(\state)\geq \underline {U}^\opt(\state) = \max_{\action\in\underline \actions_\tau^\actionsuffix} \underline Q^\opt(\state,\action).
\label{eq:lb:prefix:Uvalues}
\end{align}
\end{lemma}
\begin{proof}
The proof follows from Theorem \ref{thm:lb:ext} and from the definition of $U^\opt(\state)=\max_{\action\in\actions^\checkin} Q^\opt(\state,\action)$.
\end{proof}

\subsection{Dilatory Process as a lower bound}

A \psomdp is termed a \emph{non-drift} \psomdp, if it admits \nop actions  that leave the agent in the same state with probability one, and provides zero reward. We shall denote such actions by ``$\anop$''.
For a non-drift \psomdp, the natural choice for the suffix $\actionsuffix$ is a sequence of \nop actions $\actionsuffix=(\anop, \ldots, \anop)$. This offers an intuitive interpretation of the lower bounds in Theorem \ref{thm:lb:ext} and Lemma \ref{lemma:lb:states} as the outcome of a \emph{dilatory process} where the agent follows the optimal policy of length $\tau$, and then stops taking actions until it receives the information delivered by the next check-in.

The interpretation of the lower bounds as a dilatory process is of interest because it provides a connection between the upper bounds in Section \ref{sec:ub:bonus} (which have the agent replanning with new information after $\tau$ steps) and the lower bounds in this section (which, in the non-drift case, assume the agent pauses after $\tau$ steps to wait for new information). However, note that the optimal policy for the upper bound may differ from the optimal policy for the lower bound due to the discounting factor.

The use of \nop actions for the action suffix $\actionsuffix$ also offers a computational advantage: since \nop actions result in no transition state and no reward, the transition and reward functions can be computed as $T^\checkin(\state',\action,\state)=T^\tau(\state',\action,\state)$ and $R^{\checkin,\gamma}(\state,\action)=R^{\tau,\gamma}(\state,\action), \forall \action \in \underline {\actions}_\tau^\actionsuffix,$ where $\actionsuffix=(\anop, \ldots, \anop)$.
Since the cost of computing $T^\checkin$ and $R^\checkin$ scales exponentially with \checkin, this can result in significant computational savings.

\section{A branch-and-bound algorithm}
\label{sec:bnb}

We are now in a position to use the upper and lower bounds described in the previous sections to efficiently solve \psomdps. Specifically, we propose a \emph{branch-and-bound} algorithm that builds a sequence of increasingly tight upper and lower bounds, and uses the bounds to prune suboptimal actions. The proposed approach is described in Algorithm~\ref{alg:branch-and-bound}.
The algorithm iteratively builds upper and lower bounds for the Q-values of action prefixes for each state, and uses the bounds to discard actions whose prefix's upper bound is smaller than another action prefix's lower bound.
The key insight is to keep track of non-dominated actions for each state through the set $\algoactions{\state}$; whenever the upper bound $\overline Q^\opt (\state,\action_{0:\tau})$ for a given action prefix $\action_{0:\tau}$ (computed by using the extended Q-values presented in Section \ref{sec:ub:extended-Q-values}) is lower than the lower bound $\underline U^\opt(\state)$, as obtained via Equations \eqref{eq:lowerbound:qvalue} and \eqref{eq:lb:prefix:Uvalues}, the actions with the prefix $\action_{0:\tau}$ are discarded for that state.

\newcommand{\comm}[1]{\Comment{\textcolor{gray}{\scriptsize #1}}}
\begin{algorithm}
\small{
\begin{algorithmic}[1]
\For{$\state\in\states$}
\State $\algoactions{\state}\gets\actions$ \comm{Keep track of non-dominated actions}
\EndFor
\For{$\tau\in\{1,\ldots,\checkin\}$} \comm{Consider increasingly long action prefixes}
\State $T^{\tau}(\state',\action,\state)\gets$ Extend $T^{\tau-1}(\state',\action,\state)$  \comm{Update transitions}
\State $R^{\tau,\gamma}(\state,\action)\gets$ Extend $R^{\tau-1,\gamma}(\state,\action)$ \comm{Update  rewards} 
\If{ $\tau$ is a divisor of \checkin}
\State $\overline Q^\opt (\state,\action_{0:\tau}) \gets $ Solve \eqref{eq:bounds:ub:qvalues:BL:S} with $B=\{\tau\}$\hspace*{7ex}\phantom{.} \hspace*{8ex} but with the action prefix set restricted to be\hspace*{3ex}\phantom{.} \hspace*{8.5ex}$\{\action_{0:\tau}: \action \in \algoactions{\state}\}$ \comm{Update the upper bound}
\Else 
\State $\overline Q^\opt (\state,\action_{0:\tau}) \gets \overline Q^\opt (\state,\action_{0:\tau-1})$ \comm{Adapt previous upper bound}
\EndIf
\State $\underline U^\opt(\state)\gets$ Solve \eqref{eq:lowerbound:qvalue}, \eqref{eq:lb:prefix:Uvalues} with action prefix set\hspace*{5ex}\phantom{.} \hspace*{5ex}restricted to $\{\action_{0:\tau}: \action \in \algoactions{\state}\}$ \comm{Update the lower bound}
\For{$\state\in\states, \action\in\algoactions{\state}$}
\If{$\overline Q^\opt (\state,\action_{0:\tau}) \leq  \underline U^\opt(\state)$}\comm{Prune actions}
\State Remove all actions with prefix $\action_{0:\tau}$ from $\algoactions{\state}$ 
\EndIf
\EndFor
\EndFor
\State $Q^\opt(\state,\action), \pi^\opt(\state) \gets$ Solve \eqref{eq:mdp:qvalue}, \eqref{eq:mdp:policy} with only $\algoactions{\state}$ actions 
\end{algorithmic}
}
\caption{Branch-and-bound algorithm for \psomdps}
\label{alg:branch-and-bound}
\end{algorithm}

Compared to naively solving the \mdp version of the \psomdp with Equations \eqref{eq:mdp:qvalue} and \eqref{eq:mdp:policy}, Algorithm \ref{alg:branch-and-bound} requires solving up to $2(\checkin-1)$ additional \mdps with action sets of size upper-bounded by $A, A^2, \ldots, A^{\checkin-1}$. However, the pruning procedure can greatly reduce the size of the action set, which can result in significantly reduced computation times in practical applications, as shown next.

\section{Numerical experiments}
\label{sec:num-ex}

We assess the performance of the proposed branch-and-bound algorithm on robot navigation problems.
We consider two grid-world \psomdps, shown in Figure~\ref{fig:experiments:grid-world}. An agent must navigate to rewarding states  (shown in cyan) while avoiding obstacles (shown in purple). The agent's navigation is imperfect: when trying to drive in a given direction, the agent remains in place with 5\% probability, and drifts left or right of the desired direction with 7.5\% probability  each. We compare the time required to solve the \psomdps with Algorithm \ref{alg:branch-and-bound} with a naive approach where we formulate the \psomdp as a \mdp\gobble{ according to Definition \ref{defn:composite-action}}, and then solve it via value iteration. Figure \ref{fig:experiments} shows the time required to formulate and solve the problem with both approaches, for both problems. Due to memory limitations, the larger grid is only solved for $\checkin\leq 8$.
\begin{figure}[h]
\jvspace{-1em}
\begin{subfigure}[b]{.48\linewidth}
\centering
\includegraphics[angle=-90,totalheight=.5\linewidth]{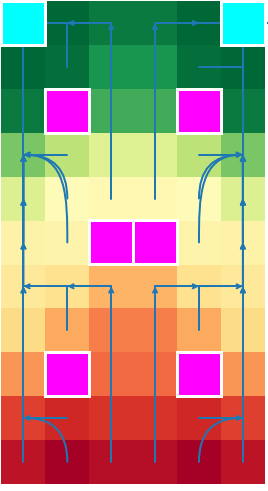}
\caption{$6\times 11$ grid}
\end{subfigure}
\begin{subfigure}[b]{.48\linewidth}
\centering
\includegraphics[angle=-90,totalheight=.33\linewidth]{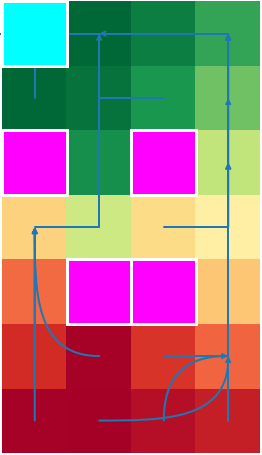}
\caption{$4\times 7$ grid}
\end{subfigure}
\caption{Grid world \psomdp problem, state values, and optimal policy. Rewarding terminal states are cyan; obstacles are purple.}
\label{fig:experiments:grid-world}
\vspace*{-1em}
\end{figure}

\begin{figure}[h]
\vspace{-1em}
\begin{subfigure}[t]{.1575\linewidth}
\vskip 0pt
\includegraphics[height=1.5\linewidth]{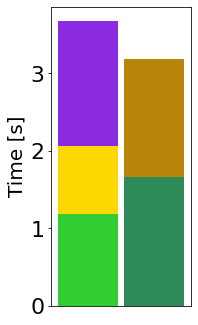}
\caption{\checkin=4}
\label{fig:experiments:medium:4}
\end{subfigure}
\begin{subfigure}[t]{.1575\linewidth}
\vskip 0pt
\includegraphics[height=1.5\linewidth]{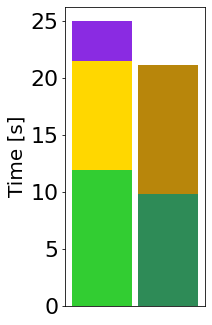}
\caption{\checkin=5}
\label{fig:experiments:medium:5}
\end{subfigure}
\begin{subfigure}[t]{.1575\linewidth}
\vskip 0pt
\includegraphics[height=1.5\linewidth]{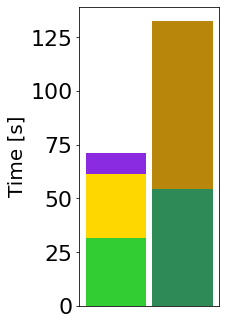}
\caption{\checkin=6}
\label{fig:experiments:medium:6}
\end{subfigure}
\begin{subfigure}[t]{.1575\linewidth}
\vskip 0pt
\includegraphics[height=1.5\linewidth]{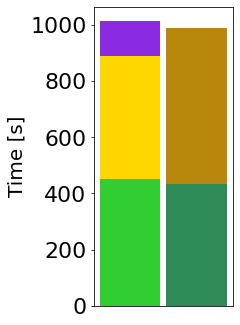}
\caption{\checkin=7}
\label{fig:experiments:medium:7}
\end{subfigure}
\begin{subfigure}[t]{.1575\linewidth}
\vskip 0pt
\includegraphics[height=1.5\linewidth]{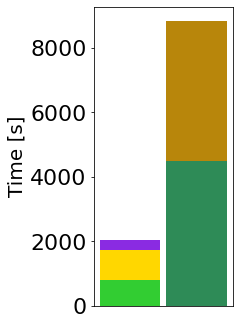}
\caption{\checkin=8}
\label{fig:experiments:medium:8}
\end{subfigure}
\begin{subfigure}[t]{.1575\linewidth}
\vskip 0pt
\includegraphics[width=\linewidth]{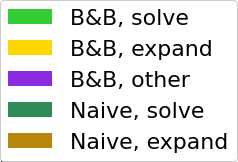}
\end{subfigure}
\begin{subfigure}[t]{.1575\linewidth}
\vskip 0pt
\includegraphics[height=1.5\linewidth]{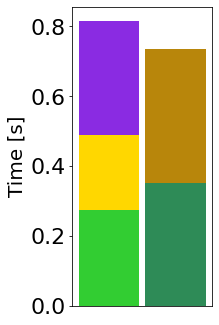}
\caption{\checkin=4}
\label{fig:experiments:small:4}
\end{subfigure}
\begin{subfigure}[t]{.1575\linewidth}
\vskip 0pt
\includegraphics[height=1.5\linewidth]{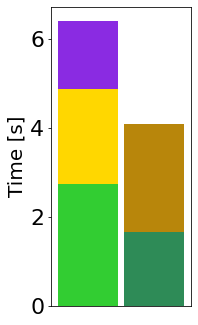}
\caption{\checkin=5}
\label{fig:experiments:small:5}
\end{subfigure}
\begin{subfigure}[t]{.1575\linewidth}
\vskip 0pt
\includegraphics[height=1.5\linewidth]{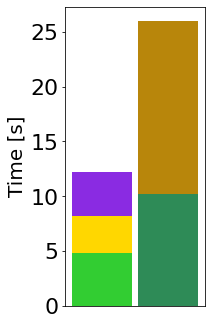}
\caption{\checkin=6}
\label{fig:experiments:small:6}
\end{subfigure}
\begin{subfigure}[t]{.1575\linewidth}
\vskip 0pt
\includegraphics[height=1.5\linewidth]{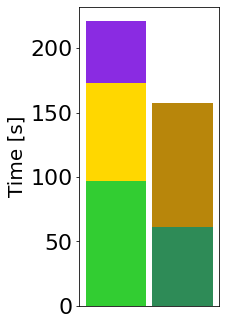}
\caption{\checkin=7}
\label{fig:experiments:small:7}
\end{subfigure}
\begin{subfigure}[t]{.1575\linewidth}
\vskip 0pt
\includegraphics[height=1.5\linewidth]{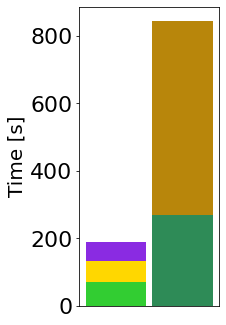}
\caption{\checkin=8}
\label{fig:experiments:small:8}
\end{subfigure}
\begin{subfigure}[t]{.1575\linewidth}
\vskip 0pt
\includegraphics[height=1.5\linewidth]{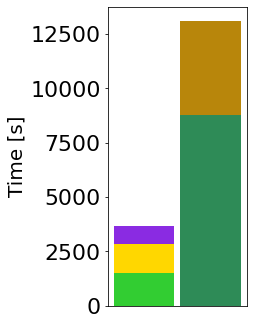}
\caption{\checkin=9}
\label{fig:experiments:small:9}
\end{subfigure}
\caption{Time required to formulate and solve a navigation \psomdp problem. Top (\subref{fig:experiments:medium:4}-\subref{fig:experiments:medium:8}): $6\times 11$ grid. Bottom (\subref{fig:experiments:small:4}-\subref{fig:experiments:small:9}): $4 \times 7$ grid.}
\label{fig:experiments}
\end{figure}

The proposed approach significantly outperforms the naive approach for
$\checkin=6$, $\checkin=8$, and (for the smaller grid) $\checkin=9$, offering a twofold to fourfold reduction in
computation time---an encouraging result that points to the branch-and-bound
approach, informed by upper and lower bounds, as a highly promising technique to make
\psomdps with large check-in periods tractable. We note that, for $\checkin=5$
and $\checkin=7$, the performance of the branch-and-bound approach is on par
or slightly worse compared to the naive approach.
This is not unexpected:  the upper bound is only
updated for $\ell$ that are divisors of $\checkin$, which results in modest
bounding and pruning when the check-in period is a prime number.
This is not a fundamental limitation of the algorithmic approach,
but rather a byproduct the simple technique used to select the
 set $B$ in Algorithm \ref{alg:branch-and-bound}.
Extending the approach to accommodate generic sets of check-in
times $B$, and devising techniques to select sets $B$ that result in tight upper bounds for general check-in periods \checkin,
 are critical directions for future research.

\section{Conclusions}
\label{sec:concl}

Planning under uncertainty\,---the crucial problem faced by robots---\,is computationally
intractable to solve in the form of completely general \pomdps.
One approach to handle this impasse is to add
simplifying assumptions, or to impose constraints, 
that afford opportunities for efficient specialized solution
methods. This is, broadly, the approach employed in the 
present paper. We have identified and examined a novel class of decision-making
problems in between \mdps and \pomdps, the former
not accounting for observation uncertainty, while the latter
being generally computationally intractable. 
The class of \psomdps model situations where the state is 
only observed periodically. 
We establish a collection of bounds for these problems by, quite intuitively,
considering cases that vary when state information is made available to
the agent. These bounds are then turned to gains in computational efficiency
via a branch-and-bound algorithm.  The paper also uncovers some intriguing nuances.
For instance, that receiving data more frequently is not always better.
Also, knowledge of \emph{when} uncertainty will be quashed can be exploited and thus be understood
to have specific value.

\section*{Acknowledgements}
Part of this work was carried out at the Jet Propulsion Laboratory (JPL), California Institute of Technology, 
 under a contract with the National Aeronautics and Space Administration (80NM0018D0004).
The work at TAMU was supported in part by NASA/Jet Propulsion Lab R\&TD Innovative Spontaneous Concept Award \#1652187, and in part by NSF Award IIS-2034097.
\copyright 2022. All rights reserved.

\bibliographystyle{IEEEtranS}
\bibliography{refs}

\end{document}